\documentclass[letterpaper, 10 pt, conference]{ieeeconf}  

\usepackage{amsmath}
\usepackage{algorithm}
\usepackage{algorithmic}

\IEEEoverridecommandlockouts                              

\usepackage{graphics} 
\usepackage{epsfig} 
\usepackage{epstopdf}
\usepackage{mathptmx} 
\usepackage{times} 
\usepackage{amsmath} 
\usepackage{amssymb}  
\usepackage{bm}
\usepackage{upgreek}
\usepackage{subfigure}
\usepackage{multirow}
\usepackage{float}
\usepackage{url}
\usepackage{textcomp}
\usepackage{color}

\newtheorem{theorem}{Theorem}
\newtheorem{lemma}{Lemma}
\newtheorem{procedure}{Procedure}
\newtheorem{definition}{Definition}
\newtheorem{note}{Note}

\newtheorem{approach}{Approach}

\title{Human Machine Co-Adaptation Model and Its Convergence Analysis}

\author{Steven W. Su$^1$ $^2$*, Yaqi Li$^1$, Kairui Guo$^2$ $^3$, Rob Duffield$^2$
\thanks{$^{*}$ The co-responding author.}
\thanks{$^{1}$ College of Medical Information and Artificial Intelligence, Shandong First Medical University
Shandong Academy of Medical Sciences.}
\thanks{$^{2}$University of Technology Sydney}
\thanks{$^{3}$South Western Sydney LHD; Liverpool Hospital.}
}

\begin{document}

\maketitle
\thispagestyle{empty}
\pagestyle{empty}

\begin{abstract}

The key to robot-assisted rehabilitation lies in the design of the human-machine interface, which must accommodate the needs of both patients and machines. Current interface designs primarily focus on machine control algorithms, often requiring patients to spend considerable time adapting. In this paper, we introduce a novel approach based on the Cooperative Adaptive Markov Decision Process (CAMDPs) model to address the fundamental aspects of the interactive learning process, offering theoretical insights and practical guidance. We establish sufficient conditions for the convergence of CAMDPs and ensure the uniqueness of Nash equilibrium points. Leveraging these conditions, we guarantee the system's convergence to a unique Nash equilibrium point. Furthermore, we explore scenarios with multiple Nash equilibrium points, devising strategies to adjust both Value Evaluation and Policy Improvement algorithms to enhance the likelihood of converging to the global minimal Nash equilibrium point. Through numerical experiments, we illustrate the effectiveness of the proposed conditions and algorithms, demonstrating their applicability and robustness in practical settings. The proposed conditions for convergence and the identification of a unique optimal Nash equilibrium contribute to the development of more effective adaptive systems for human users in robot-assisted rehabilitation.

\end{abstract}

\section{INTRODUCTION}

Robot-assisted rehabilitation involves complex sequential decision-making processes, typically modeled using Markov Decision Processes (MDPs) to optimize rehabilitation procedures \cite{stevanovic2022joint,zhang2022reinforcement,mukherjee2022survey}. Given the human feedback loop involved, reinforcement learning (RL) techniques are essential \cite{doll2012ubiquity}, with the reward prediction error (RPE) theory of dopamine (DA) function playing a key role in understanding learning mechanisms in such systems \cite{doll2012ubiquity}.

Although prior research provides valuable insights into several aspects of robot-assisted rehabilitation and human learning, there is limited understanding of the \textbf{convergence dynamics} and \textbf{numerical analyses} of these models in rehabilitation settings. This paper addresses this gap by introducing a novel approach based on the \textbf{Cooperative Adaptive Markov Decision Process (CAMDP)} model \cite{guo2024cooperative}, offering rigorous convergence analyses and numerical validation for the development of more effective rehabilitation strategies.

In the context of \textbf{Human-Machine Interfaces (HMI)}, optimizing the learning curve is crucial for enhancing the interaction between humans and machines. While machine-oriented RL research often focuses on improving cumulative reward and sample efficiency, integrating human agents into the learning loop introduces new challenges. Specifically, ensuring convergence in real-time, online settings is difficult when humans are actively involved in controlling the system. This paper takes an initial step towards developing \textbf{optimal policy algorithms} \cite{bai2019provably} that allow two agents (human and machine) to effectively switch between predefined low-level controllers to achieve convergence. This is explored through a framework based on \textbf{multi-agent cooperative game theory}.

Multi-Agent Reinforcement Learning (MARL) is divided into two categories: model-based and model-free \cite{yang2020overview}. In medical and clinical research, model-based MARL is more appropriate due to data scarcity, safety concerns, and the need for explainability \cite{moerland2020model}. Model-based RL provides an efficient approach for sequential decision-making, which is essential in robot-assisted rehabilitation, as it enables better data utilization, safety, and explainability.

The complexity of multi-agent learning, including nonstationarity and multi-dimensional objectives, often limits MARL algorithms to two-player settings. However, model-based MARL has shown advantages in handling these complexities by abstracting states and introducing temporal abstraction, allowing more efficient adaptation \cite{moerland2020model}. In our previous study \cite{guo2024cooperative}, we proposed a two-agent cooperative MDP (CaMDP) model to facilitate co-adaptation between humans and machines, particularly focusing on robot-assisted rehabilitation.

Achieving convergence in MARL is a challenging task, especially when multiple agents are involved \cite{bai2019provably}. To tackle this, we begin by exploring simpler scenarios involving two agents and finite MDPs \cite{qingji2008robot,wang2007emotion}, which helps us understand the adaptation dynamics before moving to more complex cases potentially requiring deep learning methods. Controlling model complexity in rehabilitation settings is essential to maintain interpretability, transparency, and avoid overfitting, often necessitating model reduction techniques like pruning \cite{maadi2021review,wang2022pruning}.

The CAMDP model serves dual purposes, similar to methodologies in control system design. It enables both the evaluation of the adaptation system through simulation and the high-level performance analysis of the overall system. Drawing parallels with control system stability analysis (such as Lyapunov methods), we use the reward function to analyze convergence in the co-adaptation process. Therefore, it is essential to investigate the convergence properties of the co-adaptive system within the CAMDP framework.

This study establishes \textbf{sufficient conditions for the convergence} of CAMDPs and ensures the uniqueness of Nash equilibrium points. By leveraging these conditions, we can confidently guarantee the system's convergence to a unique Nash equilibrium. We also explore cases with multiple Nash equilibrium points and develop strategies to adjust value evaluation and policy improvement algorithms to increase the likelihood of converging to the global optimal equilibrium. Through comprehensive numerical experiments, we demonstrate the robustness and applicability of these approaches in practical settings.

Our primary objective is to facilitate co-adaptation between two agents towards optimal policies. Although simulations initially involve centralized information for training, the two agents typically operate in a decentralized manner, presenting unique challenges. Prior studies in convergence analysis \cite{fleming1961convergence, lewis2012reinforcement, vamvoudakis2011non} emphasize the difficulty in formulating such games. This paper addresses these challenges by clearly defining the game structure and analyzing scenarios where agents act simultaneously or alternately. We also examine the potential for divergence due to the nonstationarity of simultaneous policy updates and explore feedback structures where one agent (e.g., the patient) has limited observations, leading to issues of partial observability.

Additionally, we investigate the conditions \cite{LOZOVANU201113398} \cite{Avrachenkov2012} to ensure convergence to the \textbf{optimal Nash equilibrium}, focusing on its uniqueness. Previous studies  \cite{sadhu2017improving, zhang2023global, leonardos2021global, song2019convergence, seierstad2014existence, chenault1986uniqueness, block2022existence, jank2003existence} have explored Nash equilibrium conditions in different contexts, including stochastic games and cooperative settings. While most studies focus on continuous games, fewer have addressed discrete-time games, which are more relevant to our investigation.

Building upon our previous work \cite{guo2024cooperative}, this paper presents a two-agent cooperative MDP model specifically designed for robot-assisted rehabilitation, defining the rules for human-machine co-adaptation and analyzing the convergence of this process. The main contributions of this work are summarized as follows:

\begin{enumerate}
	\item Establishment of sufficient conditions for CAMDP convergence and proof of the uniqueness of Nash equilibrium points, enhancing the theoretical framework for analyzing interactive learning in rehabilitation.
	\item Guaranteeing convergence to a unique Nash equilibrium, ensuring system stability and improved patient outcomes.
	\item Exploration of scenarios with multiple Nash equilibrium points and development of strategies to adjust policy improvement algorithms. These strategies, including less greedy approaches, were validated through numerical experiments, demonstrating their practical relevance and robustness.
\end{enumerate}

\section{The Co-operative MDPs Model for Co-adaptation}


Our previous research \cite{guo2024cooperative} on robot-assisted rehabilitation identified two key agents: Agent$_0$ (the patient) and Agent$_1$ (the robot), and explored their interactions within these frameworks. Due to the limited communication capacity between the human and the robot, the block diagram should follow a decentralized partially observable configuration, as shown in Fig. \ref{fig_1ob}.

Research in decentralized policies faces challenges such as partial observability and coordination issues. Approaches to address these challenges include distributed optimization \cite{boyd2011distributed}, game theory \cite{aumann1974cooperative}, and decentralized POMDPs \cite{oliehoek2016concise}, \cite{goldman2004decentralized}, \cite{pomdpssynthesis}, with recent proposals for learning schemes \cite{dobbe2017fully}.

In this paper, our primary focus is on convergence analysis, which involves complexities that could obscure an intuitive understanding of the conditions. Therefore, we adopted most of the model proposed in \cite{guo2024cooperative}, with some modifications. Specifically, we added rules for the two agents to follow based on the real-world context of robot-assisted rehabilitation. We assume that the states are directly accessible to either agent via sensors or designed observers.

\begin{figure}[ht]
  \centering
   \vspace*{-0.0in}~\\
   \hspace*{-0.0 in}
   \includegraphics[height=4.25in]{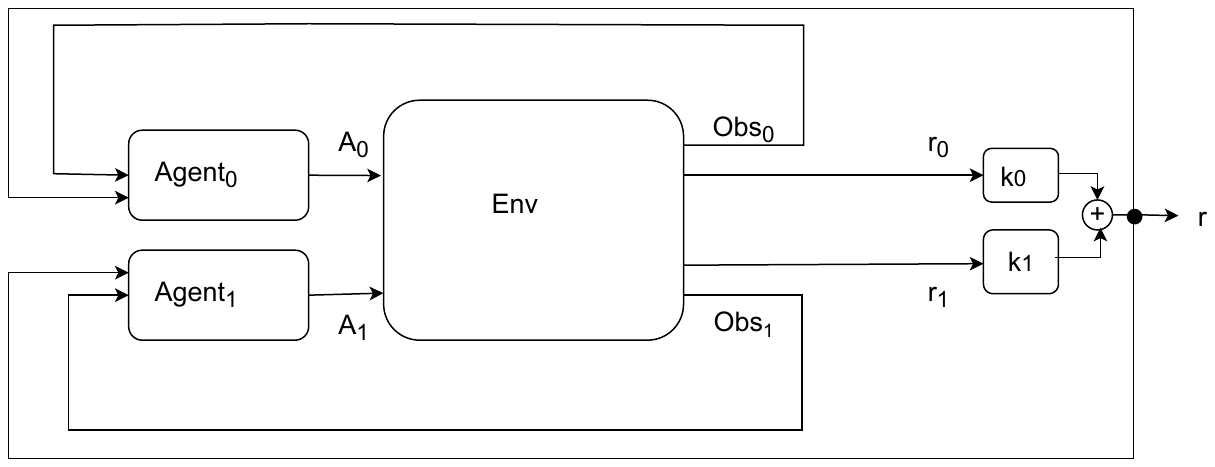}
   \vspace*{-3.0in}
\caption{The block diagram of the robot assisted rehabilitation in simulation setting.}\label{fig_1ob}
\end{figure}

In alignment with the arrangements described by \cite{goldman2004decentralized} and \cite{guo2024cooperative}, this study is confined to problems characterized by factored representations. Specifically, we examine three distinct sets of states: the first set, denoted as $s_{0_i}$ where $i \in \{1, 2, \cdots, Ns_0\}$, comprises states controlled exclusively by the patient; the second set, denoted as $s_{s_i}$ where $i \in \{1, 2, \cdots, Ns_s\}$, includes states influenced by both the robot and the patient; and the third set, denoted as $s_{1_i}$ where $i \in \{1, 2, \cdots, Ns_1\}$, consists of states controlled exclusively by the robot. We also define $\bar{Ns} = Ns_0 + Ns_s + Ns_1$.

We adopt the formal definition provided in \cite{guo2024cooperative}:

\begin{definition} \label{df_1}
A two agents MDPs is a Co-Adaptive MDPs (CAMDPs) system if the set $S$ of states can be separated into three components $S_0$, $S_1$, and $S_s$ such that:\\
$\forall$ $s_0,s'_0 \in S_0$, $\forall$ $s_s,s'_s \in S_s$, $\forall$ $s_1,s'_1 \in S_1$, we have\\
$Pr(s'_0|(s_0,s_s,s_1),a_0,a_1)=Pr(s'_0|s_0,a_0)$;\\
$Pr(s'_s|(s_0,s_s,s_1),a_0,a_1)=Pr(s'_s|s_s,a_0,a_1)$;\\
$Pr(s'_1|(s_0,s_s,s_1),a_0,a_1)=Pr(s'_1|s_1,a_1)$;\\
$Rw(s_0,a_0,a_1,(s'_0,s'_s,s'_1))=Rw(s_0,a_0,s'_0)$;\\
$Rw(s_s,a_0,a_1,(s'_0,s'_s,s'_1))=Rw(s_s,a_0,a_1,s'_s)$;\\
$Rw(s_1,a_0,a_1,(s'_0,s'_s,s'_1))=Rw(s_1,a_1,s'_1)$.

Furthermore, both the augmented transition probability $\bar P$ and the augmented reward function $\bar R$ of the CAMDPs can be represented as\\
$\bar P = P_0 \otimes P_s \otimes P_1$, where $P_0 = Pr(s'_0|s_0,a_0)$, $P_s = Pr(s'2|s2, a2)$, and $P_1 = Pr(s'2|s2, a2)$ and\\
$\bar R = R_0 \otimes R_s \otimes R_1$, where $R_0 = Rw(s_0,a_0,s'_0)$, $R_s = Rw(s_s,a_0,a_1,s'_s)$, and $R_1 = Rw(s_1,a_1,s'_1)$.
\end{definition}

\begin{note}
	In this study, to simplify our discussion, we assume the elements of the reward function matrices are all positive.  
\end{note}

This CaMDP model could be the abstraction of states and actions at the high level (see Fig. \ref{MRL_state} and Fig. \ref{MRL_input} directly from \cite{moerland2023model} ). 

\begin{figure}[ht]
\centering
\vspace*{-0.2in}~\\
\hspace*{-0.0 in}
\includegraphics[height=6.25in]{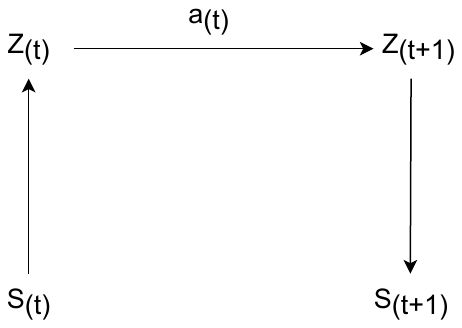}
\vspace*{-4.50in}
\caption{State abstraction: compress the state into a compact representation $z(t)$ and model the transition in this latent space (See Figure 2.4 in \cite{moerland2023model}).}\label{MRL_state}
\end{figure}

\begin{figure}[ht]
\centering
\vspace*{-0.0in}~\\
\hspace*{-0.4 in}
\includegraphics[height=8.25in]{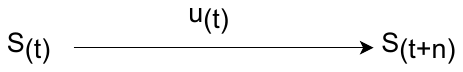}
\vspace*{-7.0in}
\caption{Temporal/action abstraction: better known as hierarchical reinforcement learning, where we learn an abstract action $u(t)$ that brings $s(t)$ to $s(t+n)$. Temporal abstraction directly
	implies multi-step prediction, as otherwise, the abstract action $u(t)$ is equal to the low-level action $a(t)$ (See Figure 2.5 in \cite{moerland2023model}).}\label{MRL_input}
	\end{figure}

When comparing two policies, we focus on their asymptotic value functions, which represent the long-term expected rewards under each policy. Accordingly, we introduce the following lemma.  

\begin{lemma} \label{Lm2} \cite{guo2024cooperative}
	For the CaMDPs model, assume the augmented probability transition matrix $\bar P$ is quasi-positive (i.e., irreducible and aperiodic stochastic). Then, for any two sub-control policies $\pi_0$ and $\pi_1$ (with $\pi = \{\pi_0, \pi_1\}$), if $\gamma \leq \gamma_0 < 1$, the value function $\bold{V}=[V(s_1), V(s_2), \cdots, V(s_{\bar N_s})]^T$ will converge to:
	$$
	\bold{V}=[\mathbf{I} - \gamma \bar{P}]^{-1} (diag(\bar{P} \bar{R}^T)) = \lim_{n \to \infty} \left( \sum_{i=0}^{n} \gamma^i \bar{P}^i \right) diag(\bar{P} \bar{R}^T),
	$$
where $\bar{R}$ is the augmented reward function, and $diag(A)$ represents the operation of extracting the diagonal elements of the square matrix $A$ and forming them into a column vector.
\end{lemma}

In this study, we investigate the process of searching for an optimal policy. First, we need to define how to compare two policies. Let $V_{ij}(s_k)$ denote the value of the policy pair $(\pi_{0i}, \pi_{1j})$, where $\pi_{0i}$ is the $i$-th policy for Agent$_0$ and $\pi_{1j}$ is the $j$-th policy for Agent$_1$, given the initial state $s_k$. We define policy $\pi = [\pi_{0i}, \pi_{1j}]$ to be better than policy $\pi' = [\pi_{0i'}, \pi_{1j'}]$ if there exists a small positive $\epsilon$, such that for $\pi \ne \pi'$:
\begin{equation} \label{comparison}
V_{ij}(s_k) - V_{i'j'}(s_{k'}) \ge \epsilon > 0, \quad \forall k, \,\, \forall k'
\end{equation}
This condition implies that for any initial state $s_k$ and $s_{k'}$, the value function under policy $\pi$ is strictly greater than the value function under policy $\pi'$.

Considering all initial states for comparison complicates the discussion. In practical scenarios, such as rehabilitation exercises, it is challenging to precisely select a particular initial state. Therefore, we prefer to analyze the value functions independent of the initial state, focusing more on the policy itself rather than the initial conditions. For example, we can use $V(i, j)$ to represent either $\max_{s_k} V_{ij}(s_k)$ or $\frac{1}{|S|} \sum_{s_k \in S} V_{ij}(s_k)$, without specifically defining the initial state, to simplify our analysis.

\begin{note}
In practical applications, and even in simulation studies, constructing a CaMDP with policies that satisfy Equation (\ref{comparison}) is often challenging. 

In numerical analysis, we can address this issue by increasing the value of $\gamma$. As $\gamma$ increases, more emphasis is placed on future rewards, which reduces the impact of the initial state.

In the limiting case, as $\gamma$ approaches $1$, the value function $\mathbf{V}$ becomes unbounded. In this scenario, rather than directly analyzing the value function $\mathbf{V}$, we consider the "average reward per transition" \cite{howard1960dynamic} for policy $\pi$, denoted by $\mathbf{g}^{\pi} = \frac{1}{n} \mathbf{V}$. Under the conditions of Lemma \ref{Lm2}, this will converge to:
\[
\mathbf{g}^{\pi} = \lim_{n \to \infty} \lim_{\gamma \to 1^{-}} \left( \frac{1}{n} \sum_{i=0}^{n} \gamma^i \bar{P}^i \right) \text{diag}(\bar{P} \bar{R}^T) = \bar{P}^{\infty} \text{diag}(\bar{P} \bar{R}^T),
\]
where all elements of the vector $\mathbf{g}^{\pi}$ are identical.

To be more specific, since $\bar{P}$ is quasi-positive (i.e., irreducible and aperiodic stochastic), the powers of $\bar{P}$ (i.e., $\bar{P}^i$ as $i \to \infty$) converge to the matrix denoted by $\bar{P}^{\infty}$, where each row is identical, representing the stationary distribution of the Markov chain.

Thus, by increasing $\gamma$, we can more accurately compare the overall value of the policies without interference from the influence of individual initial state values (for more detailed numerical simulation, see Subsection \ref{subsgamma}).

\end{note}

\section{The condition for the convergence of alternate policy improvements}

In this study, we expand on the previously defined CAMDPs by focusing on the policy convergence of two agents. We start by detailing the cooperative game settings, which form the basis of our analysis. Both agents follow the standard policy iteration approach, utilizing conventional value evaluation and policy improvement algorithms similar to those used in single-agent systems. Different with single agent setting, for two agents, we should consider two specific policy update rules: simultaneous updates and alternating updates.

Under the simultaneous rule, both agents adjust their policies concurrently, while the alternative rule entails one agent modifying its policy first and signaling the other agent to adapt once it has determined the best policy. This iterative process continues until the system either converges or oscillates in a fixed pattern. More formally, see the following two definitions:

\begin{definition}[Simultaneous Policy Update Rule] \label{SimutRule}
	In the simultaneous policy update rule, both agents \(\text{Agent}_0\) and \(\text{Agent}_1\) adjust their policies concurrently. Let \(\pi_0^k\) and \(\pi_1^k\) represent the policies of \(\text{Agent}_0\) and \(\text{Agent}_1\) at iteration \(k\). The simultaneous update rule is defined as follows:
	\[
	\pi_0^{k+1} = \arg\max_{\pi_0} \mathbb{E}_{\pi_0, \pi_1^k} \left[ V_0(s) \right]
	\]
	\[
	\pi_1^{k+1} = \arg\max_{\pi_1} \mathbb{E}_{\pi_0^k, \pi_1} \left[ V_1(s) \right]
	\]
	where \(\mathbb{E}_{\pi_0, \pi_1} \left[ V_i(s) \right]\) denotes the expected value of the value function \(V_i(s)\) for \(\text{Agent}_i\) under the policies \(\pi_0\) and \(\pi_1\).
\end{definition}

\begin{definition}[Alternating Policy Update Rule] \label{AlterRule}
	In the alternating policy update rule, one agent modifies its policy first, and then the other agent adapts its policy accordingly. Let \(\pi_0^k\) and \(\pi_1^k\) represent the policies of \(\text{Agent}_0\) and \(\text{Agent}_1\) at iteration \(k\). The alternating update rule is defined as follows:
	\begin{itemize}
		\item \textbf{Step 1: Agent$_0$ Update} \\
		\(\text{Agent}_0\) updates its policy:
		\[
		\pi_0^{k+1} = \arg\max_{\pi_0} \mathbb{E}_{\pi_0, \pi_1^k} \left[ V_0(s) \right]
		\]
		\item \textbf{Step 2: Agent$_1$ Update} \\
		\(\text{Agent}_1\) updates its policy based on the new policy of \(\text{Agent}_0\):
		\[
		\pi_1^{k+1} = \arg\max_{\pi_1} \mathbb{E}_{\pi_0^{k+1}, \pi_1} \left[ V_1(s) \right]
		\]
	\end{itemize}
	The process continues iteratively until the system either converges to a stable set of policies \((\pi_0^*, \pi_1^*)\) or exhibits a repeating pattern (oscillation).
\end{definition}

Our previous study \cite{guo2024cooperative} showed in the case of the simultaneous updating pattern, a higher likelihood of oscillations when compared to the alternative updating pattern. In order to be more focus, \textbf{we consider the  alternative updating pattern only in this study}.

\textcolor{black}{To streamline our discussion, we initially present the sufficient condition for the convergence of the alternative updating pattern via the following theorem (Theorem \ref{Th_001}). Our focus is on CAMDPs, where one agent, following the traditional Policy Improvement algorithm, can determine policy improvement based only on part of the state information. In this setup, the optimized action based on the remaining state information may align with or differ from the first. We establish that a sufficient condition for system convergence is when both policy improvements yield the same action.}

\begin{theorem} \label{Th_001}
	%
	
	Assuming that both agents of the CAMDPs, as defined in Definition \ref{df_1}, operate according to the standard policy iteration procedure (i.e., the standard Value Evaluation procedure and Policy Improvement Procedure). Furthermore, assume that the two agents adjust their policies alternately (see Definition \ref{AlterRule}). In such cases, the CAMDPs will converge to one of the Nash equilibria of the system if the following condition is satisfied:
	
	The policy updating for Agent$_0$ is only determined by states ${S_0}$ and ${S_s}$, which is not influenced by the value of state ${S_1}$. Similarly for Agent$_1$ its policy is only determined by states ${S_1}$ and ${S_s}$, which is not influenced by the value of state ${S_0}$. This implies the assumption that a conventional policy improvement algorithm based on partial information is sufficient to make decisions as if full state information were available.
\end{theorem}

\begin{proof}

	The proof strategy aims to transform the convergence problem of the two agents into a single-agent convergence problem. To achieve this, we consolidate the two agents into one entity by merging states, transfer functions, and actions. This consolidation allows us to construct an integrated policy based on all available state information, facilitating the determination of optimal actions. To establish a fully structured policy, we assume the availability of the overall possibility transformation matrix and rewards matrices within the simulation setting, or they objectively exist.
	
	However, it's important to note that from the perspective of the two agents in the CAMDPs, each agent may not possess full awareness of, or must disregard, changes in certain states. Nevertheless, if the condition specified in the theorem is met, the omission of state information does not impact the agents' policy adjustments. Essentially, this is equivalent to having full access to all the information collectively. With these foundations in place, we proceed to present a detailed proof of the theorem.
	
	
	We now treat the overall CAMDPs as a single-agent MDP, and its probability transition matrices and reward matrices under policies $\pi_0$ and $\pi_1$ can be constructed based on those of the sets of states $S_0$, $S_s$, and $S_1$. Without loss of generality, we consider the case where $\pi_1$ is fixed on its $j$-th policy, denoted as $\pi_1^j$. It is now time for the policy improvement of Agent$_0$.
	
	Suppose the current policy for Agent$_0$ is $\pi_0^m$, and under the classical policy improvement procedure, the selected policy becomes $\pi_0^n$. We denote the augmented policies for the overall system as follows:
	\begin{equation*}
		\begin{array}{cc}
			\pi^m=\{\pi_0^m, \pi_1^j\} & \pi^n=\{\pi_0^n, \pi_1^j\}.
		\end{array}
	\end{equation*}
	
	Let's clarify that the construction outlined above is based on the assumption that all state information and corresponding actions are designed within a centralized configuration. In this setup, both Agent$_0$ and Agent$_1$ adjust their policies based on all available state information. Subsequently, following the classical policy improvement procedure \cite{howard1960dynamic}, as we have chosen $\pi_0^n$ over $\pi_0^m$, we arrive at the following:
	
	$$
	r_i^{\pi^n} + \gamma \sum_{j=1}^N p_{ij}^{\pi^n} v_j^{\pi^m} \ge
	r_i^{\pi^m} + \gamma \sum_{j=1}^N p_{ij}^{\pi^m} v_j^{\pi^m}.
	$$
	where $r$ is the reward function; $p_{ij}^{\pi}$ represents the element in the $i$-th row and $j$-th column of the overall probability transition matrix under policy $\pi$; $i \in {1, 2, \cdots, N}$; and $N$ is the total number of states in the augmented system.
	
	For the two combined policies $\pi^m$ and $\pi^n$ and each state $s_i$, we have the following equations:
	
	\begin{equation} \label{eq0a}
		\begin{array}{cc}
			v_i^{\pi^m} &= r_i^{\pi^m} + \gamma \sum_{j=1}^N p_{ij}^{\pi^m} v_j^{\pi^m}, \\
			v_i^{\pi^n}& = r_i^{\pi^n} +\gamma \sum_{j=1}^N p_{ij}^{\pi^n} v_j^{\pi^n}.
		\end{array}
	\end{equation}
	
	Considering policy improvement procedure, we define for each particular state $s_i$:
	\begin{equation} \label{eq1a}
		g_i=r_i^{\pi^n} + \gamma \sum_{j=1}^N p_{ij}^{\pi^n} v_j^{\pi^m} -
		r_i^{\pi^m} - \gamma \sum_{j=1}^N p_{ij}^{\pi^m} v_j^{\pi^m}
	\end{equation}
	It is clear that $\forall i, g_i \ge 0$.
	Furthermore, under the policies $\pi^m$ and $\pi^n$, based on both equations (\ref{eq0a}) and (\ref{eq1a}), we have
	\begin{equation}
		v_i^{\pi^n}-v_i^{\pi^m}=g_i + \gamma \sum_{j=1}^{N} p_{ij}^{\pi^n} (v_j^{\pi^n}-v_j^{\pi^m}).
	\end{equation}
	Defining $\delta v_i = v_i^{\pi^n}-v_i^{\pi^m}$, we have
	\begin{equation}
		\delta v_i = g_i + \gamma \sum_{j=1}^{N} p_{ij}^{\pi^n} \delta v_j.
	\end{equation}
	As it assumed the overall CAMDPs is ergodic under all policies, we know the vector solution of the above equation is as follows:
	\begin{equation}
		\bold{\delta V} = [\bold{I} - \gamma \bold{P}^{\pi^n}]^{-1} \bold{g}.
	\end{equation}
	It can be seen that  $\bold{\delta V} \ge 0$ as the elements of $[\bold{I} - \gamma \bold{P}^{\pi^n}]^{-1}=\sum_{k=0}^{+\infty}\gamma^k {(\bold{P}^{\pi^n})}^k$ are all {\bf non-negative}. That is the policy is continiously improve, from the Agent$_0$ aspect, until a local optimal policy reached.
	
	It's essential to note that we're operating within a simulation environment, where the overall configuration follows a centralized pattern. However, in real clinical trial scenarios, agents only respond to changes in states they have access to. In other words, if the agents possess partial knowledge of the state information (indicating a partially observable system), the previous proof may not hold.  Nevertheless, the condition that the policy updating for Agent$_0$ is only determined by states ${S_0}$ and ${S_s}$, which is not influenced by the value of state ${S_1}$ assumption can still achieve the same policy improvement effect as in the centralized configuration. This conclusion finalizes the proof.
	
	
	
\end{proof}

\begin{note}
Theorem \ref{Th_001} provides a sufficient condition to ensure the same policy improvement effect as in the centralized configuration. Because this condition is sufficient, the system may still remain stable even if it is not always satisfied. In Section \ref{nu_e} (Subsection \ref{subs_convergence}), we will randomly generate 1000 CaMDP models to investigate how this sufficient condition relates to the convergence or oscillations of the CaMDP systems.

\end{note}

\section{The Condition for the Uniqueness of Nash Equilibrium} 

In the previous section, we provide a sufficient condition to ensure the convergence of the CAMDPs. However, this condition alone still cannot ensure the policies of the two agents finally converge to $optimal$ policy. The reason is the CAMDPs might have multiple Nash equilibrium, for which the two agents might converge to.

In this section, we investigate the conditions for the uniqueness of the Nash equilibrium. 
To simplify our discussion and in line with Lemma \ref{Lm2}, we assume that a sufficiently large $\gamma$ has been selected, ensuring that the values for different states differ enough to be considered approximately state-independent.
Thus, the value of each policy can be regarded as independent of the states. In the following discussion, we use $V(i, j)$ to represent the value of the $i$-th policy for Agent$_0$ and the $j$-th policy for Agent$_1$, without specifically linking it to any particular state.

\begin{lemma} \label{Th_003}
	Assume that both agents of the CAMDPs, as defined in Definition \ref{df_1}, operate according to the standard policy iteration procedure (i.e., the standard Value Evaluation procedure and Policy Improvement Procedure), and satisfy all the conditions of Theorem \ref{Th_001}. Assume the values calculated according to the standard Value Evaluation procedure are $V(i, j)$ where $i \in \{1,2,\cdots, I\}$, $j \in \{1, 2, \cdots, J\}$, $I$ is the number of the total policy of Agent$_0$ and $J$ is the number of the total policy of Agent$_1$. Further assume
	\begin{equation*}
		| V(i,j)-V(i',j') |= \epsilon > 0.
	\end{equation*}
	where $i \ne i'$ or $j \ne j'$.
	
	Then, the maximum number of Nash Equilibrium of the system can be estimated according the following three steps:
	
	1) In the value matrix $[V(i, j)]$, for the $k_{th}$ column, if it is possible to find any element in this column its value is the highest in its row, i.e., if there exist an item $V(m,k)$, such that
	$$V(m,k)=\max_{l}\{V(m, l)\}$$
	then we call this column as a row dominant column. Similarly, we can define a row of the value matrix $[V(i, j)]$ is a column dominant row. That is, for the $k_{th}$ row, if it is possible to find a element (say $V(k,n)$) in this row such that
	$$V(k,n)=\max_{l}\{V(l, n)\},$$
	then this row is a column dominant row.
	
	2) For the value matrix $[V(i, j)]$, count the total number of row dominant columns, and denote it as $N_{DC}$. Similarly, count the total number of column dominant rows, and denote it as  $N_{DR}$.
	
	3) The maximum number of Nash Equilibrium of the system can be calculated as:
	\begin{equation}\label{eqi}
		\bar N_{NE}=\min\{N_{DC}, N_{DR}\}.
	\end{equation}
	
\end{lemma}


%
%

%

\begin{proof}
	Under this assumption, the Nash equilibria should lie either in the row-dominant columns or the column-dominant rows. Therefore, it is evident that the total number of Nash equilibria must be less than or equal to either the number of row-dominant columns or the number of column-dominant rows. This conclusion proves Equation (\ref{eqi}).
\end{proof}

\begin{theorem} \label{Th_002}
	Assume that both agents of the CAMDPs, as defined in Definition \ref{df_1}, operate according to the standard policy iteration procedure (i.e., the standard Value Evaluation procedure and Policy Improvement Procedure), and satisfy all the conditions of Theorem \ref{Th_001}. Assume the values calculated according to the standard Value Evaluation procedure are $V(i, j)$ where $i \in \{1,2,\cdots, I\}$, $j \in \{1, 2, \cdots, J\}$, $I$ is the number of the total policy of Agent$_0$ and $J$ is the number of the total policy of Agent$_1$. Further assume
	\begin{equation*}
		| V(i,j)-V(i',j') |= \epsilon > 0.
	\end{equation*}
	where $i \ne i'$ or $j \ne j'$.
	Then, accodring to the operation rules, the system will reach the optimal policy within at most two round if it exist a number $m^{\star}$ (or $n^{\star}$) such that
	\begin{equation}\label{maxC}
		V(m^{\star}, j)=\max_{m \in \{1,2,\cdots, I\}} {V(m,j)} \,\,\,\,\,\,  (\forall j)
	\end{equation}
	or
	\begin{equation}\label{maxR}
		V(i, n^{\star})=\max_{n\in \{1,2,\cdots, J\}} {V(i,n)} \,\,\, \,\,\,  (\forall i).
	\end{equation}
	
\end{theorem}

\begin{proof}
	Based on Theorem \ref{Th_001}, for any initial value $V(i,j)$, when it is the due for Agent$_0$ to adjust its policy, the system will converge to the value $V(m^{\star}, j)=\max_{m \in \{1,2,\cdots, I\}} {V(m,j)}$.
	
	Next, when it is the turn for Agent$_1$ to tune its policy. Again, according to Theorem \ref{Th_001}, the system will converge to the optimal value:
	$$
	V(m^{\star}, n^{\star})=\max_{n \in \{1,2,\cdots, J\}} {V(m^{\star},n)}.
	$$
	
	As it is assumed that $V(m^{\star}, j)=\max_{m \in \{1,2,\cdots, I\}} {V(m,j)}$  ($\forall j$), we have
	$$
	V(m^{\star}, n^{\star})=\max_{\{m, n\}} {V(m,n)}.
	$$
	
	It is similar for the case when it is due for Agent$_1$.
\end{proof}

\begin{note}
	Based on Lemma \ref{Th_003}, either Condition (\ref{maxC}) or Condition (\ref{maxR}) can ensure the uniqueness of the Nash equilibrium. This is because Condition (\ref{maxC}) implies $N_{DR} = 1$, while Condition (\ref{maxR}) implies $N_{DC} = 1$.
\end{note}

\begin{note}
In Section \ref{nu_e} (Subsection \ref{subs_convergence}), we will randomly generate 1000 CaMDP models to investigate to see if these sufficient conditions of Theorem \ref{Th_002} can ensure the convergence of the CaMDP models.
\end{note}

\section{The less greedy policy improvement approach for Agent$_1$}



In the previous sections, we investigated the conditions required to ensure the two agents in the CAMDP system converge to their optimal policies under standard policy iteration procedures. However, these conditions can be overly restrictive in some cases. To address this, the following sections will introduce revised policy improvement procedures to increase the chance for the agents converge to either optimal or sub-optimal policies.

In particular, we will mainly deal with the multiple NE cases. For MARL, one major problem is still similar to its single-agent counterpart, the exploitation and exploration. In the previous sections, we considered the traditional Policy Iteration approach with a greedy search strategy, which is mainly for exploitation. Also, to reduce the uncertainty of convergence, we avoid using simultaneous-greedy methods, such as simultaneous policy improvement rather than alternative adjustment approaches. All the mentioned approaches often lead to convergence to locally optimal solutions. In the following, to avoid frequent policy switching for patient, we will use either traditional Policy Improvement algorithm or Revised Policy Improvement Procedure proposed in \cite{guo2024cooperative} for Agent$_0$, while employing a less greedy search strategy for Agent$_1$. This not only reduces the switching cost for Agent$_0$, but also helps mitigate convergence to locally optimal policies (See Algorithm \ref{alg:coop_mdp}).

\begin{algorithm} 
	\caption{Cooperative Multi-Agent MDP Optimization with Less Greediness}
	\label{alg:coop_mdp}
	\begin{algorithmic}
		\STATE Initialize policies $\pi_0$ and $\pi_1$
		\STATE Initialize value functions $V_0$ and $V_1$ (here, for cooperative manipulation, we set $V_0 = V_1$)
		\STATE Set discount factor $\gamma = 0.9$ and convergence threshold $\theta = 10^{-6}$
		\STATE Set maximum number of iterations $max\_iterations = 1000$
		\STATE Set exploration rate $\epsilon = 0.1$
		
		\FOR{iteration = 1 to $max\_iterations$}
		\STATE \textbf{Policy Evaluation for Agent$_0$}
		\REPEAT
		\STATE $\delta \leftarrow 0$
		\FOR{each state $s \in S$}
		\STATE $v \leftarrow V_0(s)$
		\STATE $V_0(s) \leftarrow \max\limits_{a \in actions_1} \sum\limits_{s', r} p(s', r|s, a) [r + \gamma V_0(s')]$
		\STATE $\delta \leftarrow \max(\delta, |v - V_0(s)|)$
		\ENDFOR
		\UNTIL{$\delta < \theta$}
		
		\STATE \textbf{Policy Improvement for Agent$_0$}
		\FOR{each state $s \in S$}
		\STATE $\pi_0(s) \leftarrow \arg\max\limits_{a \in actions_0} \sum\limits_{s', r} p(s', r|s, a) [r + \gamma V_0(s')]$
		\ENDFOR
		
		\STATE \textbf{Less Greedy Policy Evaluation for Agent$_1$}
		\REPEAT
		\STATE $\delta \leftarrow 0$
		\FOR{each state $s \in S$}
		\STATE $v \leftarrow V_1(s)$
		\STATE $V_1(s) \leftarrow \sum\limits_{s', r} p(s', r|s, \pi_1(s)) [r + \gamma V_1(s')]$
		\STATE $\delta \leftarrow \max(\delta, |v - V_1(s)|)$
		\ENDFOR
		\UNTIL{$\delta < \theta$}
		
		\STATE \textbf{Less Greedy Policy Improvement for Agent$_1$}
		\FOR{each state $s \in S$}
		\IF{rand() $\le \epsilon$}
		\STATE $\pi_1(s) \leftarrow$ random choice from $actions_1$
		\ELSE
		\STATE $\pi_1(s) \leftarrow \arg\max\limits_{a \in actions_1} \sum\limits_{s', r} p(s', r|s, a) [r + \gamma V_1(s')]$
		\ENDIF
		\ENDFOR
		
		\STATE \textbf{Check for overall convergence}
		\IF{convergence\_criteria\_met($\pi_0, \pi_1$)}
		\STATE \textbf{break}
		\ENDIF
		\ENDFOR
		
		\STATE Return final policies $\pi_0$ and $\pi_1$
	\end{algorithmic}
\end{algorithm}

This algorithm effectively balances exploration and exploitation, reducing the switching cost for Agent$_0$ and mitigating convergence to local optimal policies. The less greedy approach ensures that the agents cooperatively optimize their policies.

	\subsection{The Revised Police Improvement Procedure for Agent$_0$}

		As previously discussed, frequent switching of the control policy for Agent$_0$ (i.e., the patient) is not suitable. To reduce the switching frequency and the possibility of oscillation, the Revised Policy Improvement Procedure \ref{prd_1} proposed in \cite{guo2024cooperative} could be adopted by find a suitable threshold value $\eta$.
		
		\begin{procedure} \label{prd_1} (Revised Policy Improvement Procedure \cite{guo2024cooperative})
			\textbf{Policy Improvement}\\
			$policy\_stable \longleftarrow true$\\
			\textbf{For each} $s \in S$, $k$, \textbf{and a given} $\eta$:
			\begin{itemize}
				\item $temp \longleftarrow \pi_{k}(s)$
				\item \textbf{Under policy} $\pi_{k}$ \textbf{calculate}
				\[
				J_k= \sum_{s',r} p(s',r|s,a)[r+\gamma V(s')]
				\]
				\item \textbf{If} $\underset{a}{\max} \left( \sum_{s',r} p(s',r|s,a)[r+\gamma V(s')] \right) - J_k \ge \eta$
				\[
				\pi(s) \longleftarrow \underset{a}{\arg\max} \left( \sum_{s',r} p(s',r|s,a)[r+\gamma V(s')] \right)
				\]
				\item \textbf{If} $temp \ne \pi(s)$, \textbf{then} $policy\_stable \longleftarrow false$
			\end{itemize}
			\textbf{If} $policy\_stable$, \textbf{then stop and return} $V$ \textbf{and} $\pi$; \textbf{else go to Policy Evaluation step}.
		\end{procedure}
		
		\begin{theorem} \label{m_theorem} \cite{guo2024cooperative}
			For a CAMDP, assuming it is ergodic under all control policies, if Agent$_0$ is adjusted according to the Revised Policy Improvement (Procedure \ref{prd_1}) and Agent$_1$ performs under the control policy $\pi_1^j$, then the value loss $\boldsymbol{\delta V}$ will be less than $\eta [\boldsymbol{I} - \gamma \boldsymbol{P}^{\pi^*}]^{-1} \boldsymbol{1}$, where $\boldsymbol{P}^{\pi^*}$ is the state probability transition matrix under the optimal policy $\pi^*$, and $\boldsymbol{1}$ is the all-one vector.
		\end{theorem}

\subsection{The less greedy Policy Improvement Algorithm for Agent$_1$}	
Although Procedure \ref{prd_1} has significant potential to enhance the likelihood of convergence for the overall system and reduce the switching frequency of Agent$_0$, it still does not guarantee convergence to the global Nash Equilibrium (NE). To address this, we propose a less greedy approach (see Algorithm \ref{alg:coop_mdp}) for Agent$_1$ to improve the system’s chances of reaching the global NE.

		\subsection{Reducing Switching Frequency through Model Simplification}

As discussed in \cite{yu2019convergent} and \cite{bai2019provably}, the upper bound of the local switching cost in a Markov Decision Process (MDP) is \( N_a N_s \), where \( N_a \) represents the number of actions and \( N_s \) represents the number of states. To lower the switching rate between policies, one potential approach is to reduce the model's complexity by minimizing the number of states and actions.

In this context, we utilize the reward function to guide the complexity reduction of CaMDPs. Specifically, we employ a "pruning" strategy. Unlike existing techniques \cite{bicego2003sequential}, \cite{lin2017runtime}, \cite{tan2021towards}, \cite{gupta2020learning}, and \cite{xu2021survey}, our focus is on reducing the size of the policy-oriented state set, which is guided by sensitivity analysis concerning the value function.

Before proceeding, we introduce the notation for the policy via a simple CaMDP model. Assume the system consists of three states: \(s_0\), \(s_s\), and \(s_1\), and two actions: \(a_0\) and \(a_1\). The policies are defined as \(\pi_0(s_0, s_s)\) for action \(a_0\), and \(\pi_1(s_1, s_s)\) for action \(a_1\).
	
	Specifically, for example,  we define the state-action pairs for Agent$_0$ as:
	\begin{equation*}
		\left[
		\begin{array}{ccccc}
			State: &\{s_{0_0}, s_{s_0}\}& \{s_{0_0}, s_{s_1}\}& \{s_{0_1}, s_{s_0}\}& \{s_{0_1}, s_{s_1}\} \\
			a_0:& 1& 1& 1& 1
		\end{array}
		\right].
	\end{equation*}
	
	Similarly, for Agent$_1$:
	\begin{equation*}
		\left[
		\begin{array}{ccccc}
			State: &\{s_{1_0}, s_{s_0}\}& \{s_{1_0}, s_{s_1}\}& \{s_{1_1}, s_{s_0}\}& \{s_{1_1}, s_{s_1}\} \\ a_1:& 1& 1& 0& 0
		\end{array}
		\right].
	\end{equation*}
	
For simplicity, when no ambiguity arises, we denote the policy as: \(\pi_0 = [1\, 1\, 1\, 1]\) and \(\pi_1 = [1\, 1\, 0\, 0]\), or even \([1\, 1\, 1\, 1]\) \([1\, 1\, 0\, 0]\) for short.

It is important to emphasise that in order to evaluate a policy, we must consider both the policies of Agent$_0$ and Agent$_1$ simultaneously, as the value evaluation algorithm requires the policies of both agents.

We discuss two cases of policy simplification: 

i) Policy Pruning:\\ 
For example, to reduce the policy \([1\, 1\, 1\, 1]\) for Agent\(_0\), we evaluate the value function for the following policy set:
$\{[1\,1\,1\,1] [0\, 0\, 0\, 0]; [1\,1\,1\,1] [0\, 0\, 0\, 1]; [1\,1\,1\,1] [0\, 0\, 1\, 0]; \cdots;$ $ \cdots; [1\,1\,1\,1] [1\, 1\, 1\, 1]\}$.
If the maximum value from this set is low, i.e.,
\[
\max \left\{ V([1\, 1\, 1\, 1], [0\, 0\, 0\, 0]), \dots, V([1\, 1\, 1\, 1], [1\, 1\, 1\, 1]) \right\} \leq \epsilon,
\]
we can ignore the policy \([1\, 1\, 1\, 1]\) during the Policy Improvement process.

ii) State Reduction in Policy:\\ 
Since the policy is a function of states, in order to reduce the policy by observing fewer states, we must investigate all policies associated with the states to be removed. In this example, for Agent\(_0\), to eliminate the dependency on state \(s_s\), we need to evaluate the values of all policies for Agent\(_0\), where actions differ based on state \(s_s\), while \(s_0\) remains the same. The relevant policy set is denoted as \(\Pi_0^{S_s}\), which includes:
\begin{equation}
	\begin{aligned}
		\Pi_0^{S_s} = & \{ [0\, 0\, 0\, 1], [0\, 0\, 1\, 0], [0\, 1\, 0\, 0], [0\, 1\, 0\, 1], [0\, 1\, 1\, 0], [0\, 1\, 1\, 1], \\
		& [1\, 0\, 0\, 0], [1\, 0\, 0\, 1], [1\, 0\, 1\, 0], [1\, 0\, 1\, 1], [1\, 1\, 0\, 1], [1\, 1\, 1\, 0] \}.
	\end{aligned}
\end{equation}

If the values for all these policies combined with the policies of Agent\(_1\) are low, i.e.,
\begin{equation}
	\bar{V}^{\Pi_0^{S_s}} = \max_{\pi \in \Pi_0^{S_s}} V^\pi \leq \epsilon,
\end{equation}
we can safely simplify the policy by removing the dependency on state \(s_s\), thus reducing \(\pi_0(s_0, s_s)\) to \(\pi_0(s_0)\).

Alternatively, if the difference in values among all the policies in \(\Pi_0^{S_s}\) is small, i.e., for a small \(\epsilon\),
\begin{equation}
	\Delta V^{\Pi_0^{S_s}} = \max_{\pi \in \Pi_0^{S_s}} V^\pi - \min_{\pi \in \Pi_0^{S_s}} V^\pi \leq \epsilon,
\end{equation}
we can also simplify the policy \(\pi_0(s_0, s_s)\) to \(\pi_0(s_0)\).

Based on the above discussion, we summarize the following two approaches for reducing complexity of the CaMDPs.

	\begin{approach}\label{Pr_mrd1}
		\textbf{(Direct Pruning Based on Value Functions)}
		
		This approach can be applied both online and offline.
		
		\textbf{Online Implementation:} We can estimate the potential value loss online by adopting a method similar to the one described in the proof of Theorem \ref{m_theorem}. Specifically, we evaluate the effect of removing a subset of policies that are either difficult for Agent$_0$ (the patient) to execute or infeasible to switch between due to practical constraints. This pruning is based on real-time updates to value functions during learning.
		
		\textbf{Offline Implementation:} Assuming the availability of the transition probability matrices and reward functions for the system, we use standard value evaluation methods. By applying Lemma \ref{Lm2}, we compute both the discounted and undiscounted value functions. Policies with low value estimates, which form the bottom set, are pruned from the policy set to reduce complexity in decision-making.
		
	\end{approach}

%
%
%
%

\begin{approach}\label{Pr_mrd32}
	(State-selection based approach)
	
	This approach is designed for offline implementation. To simplify our discussion, we only consider the policy reduction for Agent$_0$, and assume the associated policies for Agent$_1$ have been augmented already.
	
	 There are two ways to reduce the number of states. The first one as we discussed before, we can select all the policies to be removed due to the violation of the constraint of the reduction of the states. 
	 
	 The second approach is we directly select a specific subset of states $S_\delta =\{s_{\delta 1}, s_{\delta 2}, \cdots, s_{\delta m} \}$  ( \( S_\delta  \subset S \) ), we collect all the associated policies with this specifically selected subset of the states $\Pi^{S_\delta}$,  we then compute the value associated with these policies. 
	  That is, we use the value function evaluation procedure to calculate the state-value functions \( V^{\Pi^{S_\delta}} \) for the set of policies $\Pi^{S_\delta}$. The value loss \( \Delta V^{\Pi^{S_\delta}} \) due to the shink of the policy set can be expressed as:
	\[
	\Delta V^{\Pi^{S_\delta}} = \max_{\pi \in \Pi^{S}} V^\pi - \max_{\pi \in \Pi^{S_\delta}}V^{\pi}.
	\]
	If the value loss \( \Delta V^{\Pi^{S_\delta}} \) is acceptably small, the policy can be simplified to the subset \( \Pi^{S_\delta} \), resulting in the modified policy $\pi(s_\delta)$.

\end{approach}

\begin{note}
It should be noted the construction of the policy subset $\Pi^{S_\delta}$, in some cases, is easy. We can simply add the constraint to original policies set. For example, for the selection of only $s_0$, i.e., $\pi_0(s_0)$, if we use 
$[p_1 \, p_2 \, p_2 \, p_4]$ 
we can simply impose the following constraint to the policy of Agent$_0$:
$$
\left\{
\begin{aligned}
	p_1 &= p_2 \\
	p_3 &= p_4.
\end{aligned}
\right.
$$

On the other hand, for the selection of only $s_s$, i.e., $\pi_0(s_s)$, we can simply impose the following constraint to the policy of Agent$_0$:
$$
\left\{
\begin{aligned}
	p_1 &= p_3 \\
	p_2 &= p_4.
\end{aligned}
\right.
$$
\end{note}

%
%
%

			\section{Numeral Analysis} \label{nu_e}
			
			\subsection{Comparable of value functions by increasing discount factor $\gamma$} \label{subsgamma}
			
%
			In order to compare the value of the policies without considering the state influence, in Lemma \ref{Lm2}, we prove that for a particular policy, increasing the discount factor $\gamma$ can make the values of different states as close as desired. To demonstrate this, we randomly generated a CaMDP with 8 states. By increasing the discount factor $\gamma$, we observe that the values of each policy for different states increase, but the overlapping of values due to different states is diminished. Below, we list the values of two different policies for all 8 states at selected values of $\gamma$, as shown in Table \ref{tablevf} and Fig. \ref{gamma}.

			\begin{table*}[h]
				\centering
				\caption{Value comparison for two policies with different $\gamma$ values}
				\label{tablevf}
				\begin{tabular}{|c|c|c|c|c|c|c|c|c|}
					\hline
					State & P1 (Policy 1) ($\gamma = 0.5$) & P2 (Policy 2) ($\gamma = 0.5$) & P1 ($\gamma = 0.75$) & P2 ($\gamma = 0.75$) & P1 ($\gamma = 0.95$) & P2 ($\gamma = 0.95$) & P1 ($\gamma = 0.998$) & P2 ($\gamma = 0.998$) \\ \hline
					1     & 0.3340                     & 0.2832                     & 0.7279                     & 0.6735                     & 3.9038                     & 3.7969                     & 99.2792                    & 97.5156                    \\ \hline
					2     & 0.3416                     & 0.3202                     & 0.7354                     & 0.7208                     & 3.9113                     & 3.8541                     & 99.2867                    & 97.5753                    \\ \hline
					3     & 0.4179                     & 0.3590                     & 0.8093                     & 0.7477                     & 3.9835                     & 3.8703                     & 99.3586                    & 97.5888                    \\ \hline
					4     & 0.4238                     & 0.4178                     & 0.8217                     & 0.8018                     & 3.9959                     & 3.9189                     & 99.3791                    & 97.6360                    \\ \hline
					5     & 0.4283                     & 0.4191                     & 0.8243                     & 0.8055                     & 4.0023                     & 3.9265                     & 99.3917                    & 97.6447                    \\ \hline
					6     & 0.4304                     & 0.4201                     & 0.8297                     & 0.8141                     & 4.0158                     & 3.9399                     & 99.3938                    & 97.6581                    \\ \hline
					7     & 0.4349                     & 0.4286                     & 0.8369                     & 0.8188                     & 4.0268                     & 3.9421                     & 99.4048                    & 97.6617                    \\ \hline
					8     & 0.4409                     & 0.4322                     & 0.8407                     & 0.8250                     & 4.0380                     & 3.9526                     & 99.4081                    & 97.6726                    \\ \hline
				\end{tabular}
			\end{table*}

			\subsection{Conditions validation of Convergence and the uniqueness of the NE point}  \label{subs_convergence}
			This section will present the numerical analysis of the proposed conclusion regarding convergence and the uniqueness of the NE point. 
			
			First, we randomly generated $1000$ CaMDPs. For each system of these CaMDPs, we checked three conditions: (1) the maximum value appears in the same column or row (i.e., Condition (\ref{maxC}) or (\ref{maxR}) of Theorem \ref{Th_002}); (2) the observability condition (i.e., the condition of Theorem \ref{Th_001}). (3) Under different initial conditions the system converges to the global optimal NE.  
			 
		In 1000 simulations, there are 150 systems satisfy Condition (1), 516 systems satisfy Condition (2),  and 652 systems satisfy Condition (3). Specifically, all the systems which satisfy both Conditions (1) and (2) satisfy Condition (3), which is consistent with the conclusion of Theorem \ref{Th_002}.
	
		\subsection{$\epsilon$-greedy} 	
	To test our proposed new policy improvement strategy, we randomly selected one of the CaMDPs ($M=<S, \, A_0,\, A_1, P_0, P_1,R_0, R_1>$ ), which did not satisfy Condition (1) and (3), but satisfy Condition (2), with details as follows:

			The probability transition matrices under different control actions for each subsystem are as follows:
			
			\begin{equation*}
				P_0(S_0, a_{0_0}, S_0) = \begin{bmatrix} 
					0.72896067  & 0.27103933 \\ 
					0.95167994  & 0.04832006 
				\end{bmatrix};
			\end{equation*}
			
			\begin{equation*}
				P_s(S_s, a_{0_0}, a_{1_0}, S_s) = \begin{bmatrix} 
					0.66489771  & 0.33510229 \\ 
					0.51544335  & 0.48455665 
				\end{bmatrix};
			\end{equation*}
			
			\begin{equation*}
				P_0(S_0, a_{0_1}, S_0) = \begin{bmatrix} 
					0.15320242  & 0.84679758 \\ 
					0.55851098  & 0.44148902 
				\end{bmatrix};
			\end{equation*}
			
			\begin{equation*}
				P_s(S_s, a_{0_0}, a_{1_1}, S_s) = \begin{bmatrix} 
					0.07046136  & 0.92953864 \\ 
					0.52137167  & 0.47862833 
				\end{bmatrix};
			\end{equation*}
			
			\begin{equation*}
				P_s(S_s, a_{0_1}, a_{1_0}, S_s) = \begin{bmatrix} 
					0.56727427  & 0.43272573 \\ 
					0.11531405  & 0.88468595 
				\end{bmatrix};
			\end{equation*}
			
			\begin{equation*}
				P_1(S_1, a_{1_0}, S_1) = \begin{bmatrix} 
					0.35013916  & 0.64986084 \\ 
					0.37319646  & 0.62680354 
				\end{bmatrix};
			\end{equation*}
			
			\begin{equation*}
				P_s(S_s, a_{0_1}, a_{1_1}, S_s) = \begin{bmatrix} 
					0.65019582  & 0.34980418 \\ 
					0.41909603  & 0.58090397 
				\end{bmatrix};
			\end{equation*}
			
			\begin{equation*}
				P_1(S_1, a_{1_1}, S_1) = \begin{bmatrix} 
					0.47227529  & 0.52772471 \\ 
					0.39457278  & 0.60542722 
				\end{bmatrix}.
			\end{equation*}

The reward functions are as follows:

\begin{equation*}
	R_0(S_0,a_{0_0},S_0)=\left[\begin{array}{cc}
		0.25561406  & 0.67130943 \\
		0.59900591  & 0.71733215
	\end{array}\right];
\end{equation*}

\begin{equation*}
	R_0(S_0,a_{0_1},S_0)=\left[\begin{array}{cc}
		0.93734953  & 0.35180977 \\
		0.25363410  & 0.40247251
	\end{array}\right];
\end{equation*}

\begin{equation*}
	R_s(S_s,a_{0_0},a_{1_0},S_s)=\left[\begin{array}{cc}
		0.39837292  & 0.77088097 \\
		0.76475098  & 0.28385938
	\end{array}\right];
\end{equation*}

\begin{equation*}
	R_s(S_s,a_{0_0},a_{1_1},S_s)=\left[\begin{array}{cc}
		0.18954219  & 0.47125096 \\
		0.33480604  & 0.73473504
	\end{array}\right];
\end{equation*}

\begin{equation*}
	R_s(S_s,a_{0_1},a_{1_0},S_s)=\left[\begin{array}{cc}
		0.18910712  & 0.33110407 \\
		0.84422842  & 0.61502403
	\end{array}\right];
\end{equation*}

\begin{equation*}
	R_s(S_s,a_{0_1},a_{1_1},S_s)=\left[\begin{array}{cc}
		0.88526408  & 0.97655302 \\
		0.83690859  & 0.18082463
	\end{array}\right];
\end{equation*}

\begin{equation*}
	R_1(S_1,a_{1_0},S_1)=\left[\begin{array}{cc}
		0.74651072  & 0.72407057 \\
		0.40610780  & 0.98937985
	\end{array}\right];
\end{equation*}

\begin{equation*}
	R_1(S_1,a_{1_1},S_1)=\left[\begin{array}{cc}
		0.45049928  & 0.37380843 \\
		0.70962861  & 0.08245855
	\end{array}\right].
\end{equation*}
			
			When the original policy of  is $\{0,0,0,0\}$, and the agent$_1$ is $\{1,0,0,0\}$, then according to the standard policy improvement algorithm the system converges to local NE point  $\{1,1,0,0\}$ and $\{1,0,1,0 \}$, which is the second highest value (9.81). Then we adapt the proposed method to this system, after the several round of policy $\epsilon$-greedy it converge to the global maximum NE point (9.99).
						
			\begin{figure}[ht]
					\vspace{-2.5cm}
					\hspace*{-2.0cm}
				\centering{\includegraphics[width=0.75\textwidth]{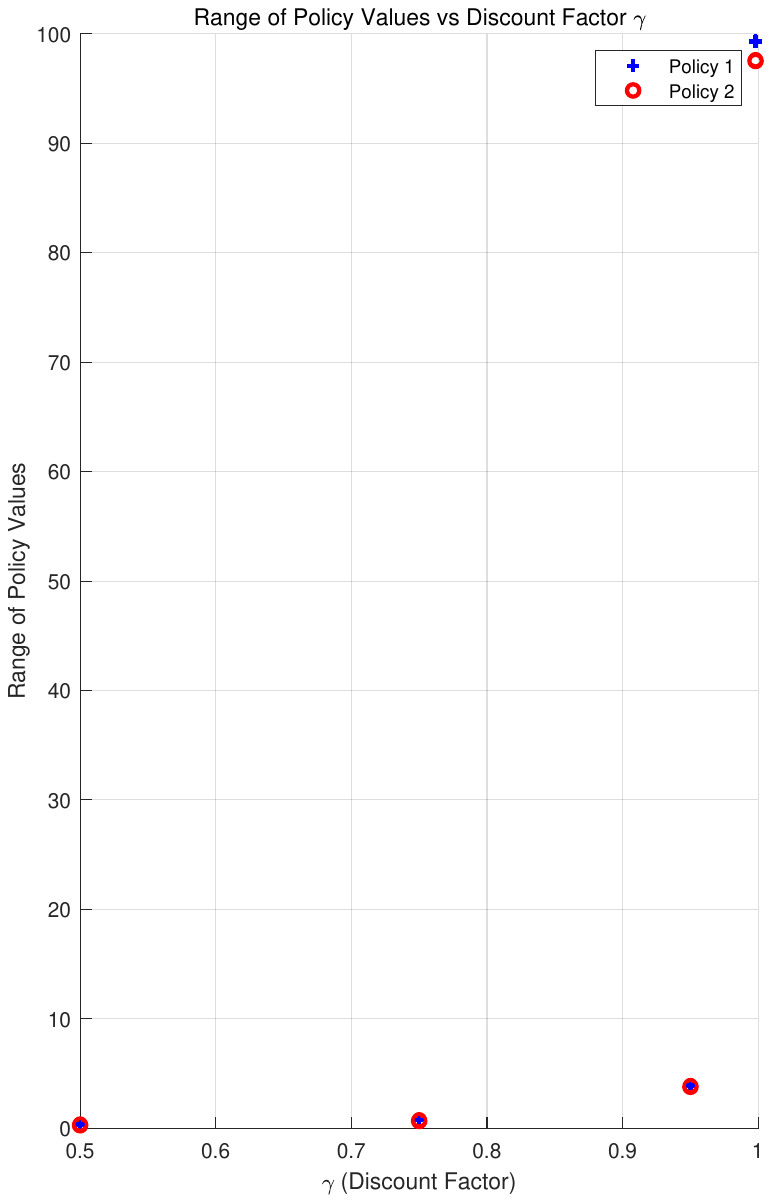}}\\
				\vspace{-3.2cm}
			\caption{Value vs Discount Factor $\gamma$.}\label{gamma}
			\end{figure}
			
%
			
			\subsection{Policy switching frequency Reduction via policy simplification}
			
			The numerical study in this subsection focuses on the Approach \ref{Pr_mrd32}. All the probability transition matrices and the reward
			functions for each policy $\pi_i$ are all the same as before.
			
			In order to reduce the switching frequency of the policy of the patient (i.e., Agent$_0$), we assume only one state is applied for the design of the control policy of Agent$_0$ first. In this case,  when only the policy is a feedback of the state $s_s$ (i.e., $\pi_0(s_s)$ by $p_1=p_3, p_2=p_4$) then, the number of policies for Agent$_0$ has been reduced from $16$ to $4$, and the best discounted value only reduced from $9.99$ to $ 9.05$. Moreover, when only the policy of Agent$_0$ is a feedback of the state $s_0$ (i.e., $\pi_0(s_0)$ by $p_1=p_2, p_3=p_4$), the number of policies for Agent$_0$ has also been reduced from $16$ to $4$ and the best discounted value does not change much ($9.81$). This clearly shows that the proposed strategy can sufficiently reduce the number of policy switching with the acceptable best discounted value.

\section{Conclusion}
\label{sec:Conclusion}
This study introduced a novel Cooperative Adaptive Markov Decision Process (CAMDP) framework to enhance the interactive learning process in robot-assisted rehabilitation. By formulating a dual-agent learning model, we established theoretical conditions for convergence and provided insights into ensuring the uniqueness of Nash equilibrium points. Our findings demonstrate that under specific conditions, the proposed CAMDP framework guarantees convergence to a stable equilibrium, reducing the risk of policy oscillations and enhancing system robustness.

We explored various policy improvement strategies, particularly focusing on alternative policy updating patterns and less greedy approaches to enhance adaptability while mitigating convergence to suboptimal equilibria. The revised policy improvement strategy for Agent$_0$ (patient) aims to reduce unnecessary switching frequency, ensuring a smoother rehabilitation process. Additionally, we implemented a modified policy update method for Agent$_1$ (robot), allowing for a more adaptive and patient-centered rehabilitation experience.

Our numerical analysis validated the effectiveness of these strategies by examining randomly generated CAMDPs. The results confirmed that the proposed convergence conditions significantly improve the likelihood of reaching the global optimal Nash equilibrium while reducing policy oscillations. Furthermore, the integration of an $\epsilon$-greedy approach for policy updates demonstrated its potential in enhancing learning efficiency while balancing exploration and exploitation.

Overall, this study contributes to the development of more adaptive and cooperative reinforcement learning models for human-machine interaction in rehabilitation settings. Future work will focus on extending these findings to more complex environments and real-world applications, incorporating additional patient-specific factors to refine the adaptability of the proposed CAMDP framework. Additionally, further studies may explore deep reinforcement learning techniques to enhance the scalability and generalizability of the model in diverse rehabilitation scenarios.

%

%






\bibliographystyle{IEEEtran}
\bibliography{main}

\begin{thebibliography}{10}
\providecommand{\url}[1]{#1}
\csname url@samestyle\endcsname
\providecommand{\newblock}{\relax}
\providecommand{\bibinfo}[2]{#2}
\providecommand{\BIBentrySTDinterwordspacing}{\spaceskip=0pt\relax}
\providecommand{\BIBentryALTinterwordstretchfactor}{4}
\providecommand{\BIBentryALTinterwordspacing}{\spaceskip=\fontdimen2\font plus
\BIBentryALTinterwordstretchfactor\fontdimen3\font minus \fontdimen4\font\relax}
\providecommand{\BIBforeignlanguage}[2]{{%
\expandafter\ifx\csname l@#1\endcsname\relax
\typeout{** WARNING: IEEEtran.bst: No hyphenation pattern has been}%
\typeout{** loaded for the language `#1'. Using the pattern for}%
\typeout{** the default language instead.}%
\else
\language=\csname l@#1\endcsname
\fi
#2}}
\providecommand{\BIBdecl}{\relax}
\BIBdecl

\bibitem{stevanovic2022joint}
M.~Stevanovic, T.~Valkeap{\"a}{\"a}, E.~Weiste, and C.~Lindholm, ``Joint decision making in a mental health rehabilitation community: the impact of support workers’ proposal design on client responsiveness,'' \emph{Counselling Psychology Quarterly}, vol.~35, no.~1, pp. 129--154, 2022.

\bibitem{zhang2022reinforcement}
R.~Zhang, Q.~Lv, J.~Li, J.~Bao, T.~Liu, and S.~Liu, ``A reinforcement learning method for human-robot collaboration in assembly tasks,'' \emph{Robotics and Computer-Integrated Manufacturing}, vol.~73, p. 102227, 2022.

\bibitem{mukherjee2022survey}
D.~Mukherjee, K.~Gupta, L.~H. Chang, and H.~Najjaran, ``A survey of robot learning strategies for human-robot collaboration in industrial settings,'' \emph{Robotics and Computer-Integrated Manufacturing}, vol.~73, p. 102231, 2022.

\bibitem{doll2012ubiquity}
B.~B. Doll, D.~A. Simon, and N.~D. Daw, ``The ubiquity of model-based reinforcement learning,'' \emph{Current opinion in neurobiology}, vol.~22, no.~6, pp. 1075--1081, 2012.

\bibitem{guo2024cooperative}
K.~Guo, A.~Cheng, Y.~Li, J.~Li, R.~Duffield, and S.~W. Su, ``Cooperative markov decision process model for human--machine co-adaptation in robot-assisted rehabilitation,'' \emph{Knowledge-Based Systems}, vol. 291, p. 111572, 2024.

\bibitem{bai2019provably}
Y.~Bai, T.~Xie, N.~Jiang, and Y.-X. Wang, ``Provably efficient q-learning with low switching cost,'' \emph{arXiv preprint arXiv:1905.12849}, 2019.

\bibitem{yang2020overview}
Y.~Yang and J.~Wang, ``An overview of multi-agent reinforcement learning from game theoretical perspective,'' \emph{arXiv preprint arXiv:2011.00583}, 2020.

\bibitem{moerland2020model}
T.~M. Moerland, J.~Broekens, and C.~M. Jonker, ``Model-based reinforcement learning: A survey,'' \emph{arXiv preprint arXiv:2006.16712}, 2020.

\bibitem{qingji2008robot}
G.~Qingji, W.~Kai, and L.~Haijuan, ``A robot emotion generation mechanism based on pad emotion space,'' in \emph{International Conference on Intelligent Information Processing}.\hskip 1em plus 0.5em minus 0.4em\relax Springer, 2008, pp. 138--147.

\bibitem{wang2007emotion}
G.~Wang, Z.~Wang, S.~Teng, Y.~Xie, and Y.~Wang, ``Emotion model of interactive virtual humans on the basis of mdp,'' \emph{Frontiers of Electrical and Electronic Engineering in China}, vol.~2, no.~2, pp. 156--160, 2007.

\bibitem{maadi2021review}
M.~Maadi, H.~Akbarzadeh~Khorshidi, and U.~Aickelin, ``A review on human--ai interaction in machine learning and insights for medical applications,'' \emph{International journal of environmental research and public health}, vol.~18, no.~4, p. 2121, 2021.

\bibitem{wang2022pruning}
L.~Wang, W.~Huang, M.~Zhang, S.~Pan, X.~Chang, and S.~W. Su, ``Pruning graph neural networks by evaluating edge properties,'' \emph{Knowledge-Based Systems}, vol. 256, p. 109847, 2022.

\bibitem{fleming1961convergence}
W.~H. Fleming, ``The convergence problem for differential games,'' \emph{J. Math. Anal. Appl}, vol.~3, no.~1, pp. 102--116, 1961.

\bibitem{lewis2012reinforcement}
F.~L. Lewis, D.~Vrabie, and K.~G. Vamvoudakis, ``Reinforcement learning and feedback control: Using natural decision methods to design optimal adaptive controllers,'' \emph{IEEE Control Systems Magazine}, vol.~32, no.~6, pp. 76--105, 2012.

\bibitem{vamvoudakis2011non}
K.~G. Vamvoudakis and F.~L. Lewis, ``Non-zero sum games: Online learning solution of coupled hamilton-jacobi and coupled riccati equations,'' in \emph{2011 IEEE International Symposium on Intelligent Control}.\hskip 1em plus 0.5em minus 0.4em\relax IEEE, 2011, pp. 171--178.

\bibitem{LOZOVANU201113398}
D.~Lozovanu, S.~Pickl, and E.~Kropat, ``Markov decision processes and determining nash equilibria for stochastic positional games,'' \emph{IFAC Proceedings Volumes}, vol.~44, no.~1, pp. 13\,398--13\,403, 2011, 18th IFAC World Congress.

\bibitem{Avrachenkov2012}
K.~Avrachenkov, L.~Cottatellucci, and L.~Maggi, ``Cooperative markov decision processes: Time consistency, greedy players satisfaction, and cooperation maintenance,'' \emph{International Journal of Games Theory}, vol.~42, pp. 239--262, 02 2012.

\bibitem{sadhu2017improving}
A.~K. Sadhu and A.~Konar, ``Improving the speed of convergence of multi-agent q-learning for cooperative task-planning by a robot-team,'' \emph{Robotics and Autonomous Systems}, vol.~92, pp. 66--80, 2017.

\bibitem{zhang2023global}
Y.~Zhang, G.~Qu, P.~Xu, Y.~Lin, Z.~Chen, and A.~Wierman, ``Global convergence of localized policy iteration in networked multi-agent reinforcement learning,'' \emph{Proceedings of the ACM on Measurement and Analysis of Computing Systems}, vol.~7, no.~1, pp. 1--51, 2023.

\bibitem{leonardos2021global}
S.~Leonardos, W.~Overman, I.~Panageas, and G.~Piliouras, ``Global convergence of multi-agent policy gradient in markov potential games,'' \emph{arXiv preprint arXiv:2106.01969}, 2021.

\bibitem{song2019convergence}
X.~Song, T.~Wang, and C.~Zhang, ``Convergence of multi-agent learning with a finite step size in general-sum games,'' \emph{arXiv preprint arXiv:1903.02868}, 2019.

\bibitem{seierstad2014existence}
A.~Seierstad, ``Existence of open loop nash equilibria in certain types of nonlinear differential games,'' \emph{Optimization Letters}, vol.~8, pp. 247--258, 2014.

\bibitem{chenault1986uniqueness}
L.~A. Chenault, ``On the uniqueness of nash equilibria,'' \emph{Economics Letters}, vol.~20, no.~3, pp. 203--205, 1986.

\bibitem{block2022existence}
D.~Block and M.~R. Rivas, ``The existence and uniqueness of a nash equilibrium in mean field game theory,'' \emph{arXiv preprint arXiv:2210.10117}, 2022.

\bibitem{jank2003existence}
G.~Jank and H.~Abou-Kandil, ``Existence and uniqueness of open-loop nash equilibria in linear-quadratic discrete time games,'' \emph{IEEE Transactions on Automatic Control}, vol.~48, no.~2, pp. 267--271, 2003.

\bibitem{boyd2011distributed}
S.~Boyd, N.~Parikh, and E.~Chu, \emph{Distributed optimization and statistical learning via the alternating direction method of multipliers}.\hskip 1em plus 0.5em minus 0.4em\relax Now Publishers Inc, 2011.

\bibitem{aumann1974cooperative}
R.~J. Aumann and J.~H. Dreze, ``Cooperative games with coalition structures,'' \emph{International Journal of game theory}, vol.~3, no.~4, pp. 217--237, 1974.

\bibitem{oliehoek2016concise}
F.~A. Oliehoek and C.~Amato, \emph{A concise introduction to decentralized POMDPs}.\hskip 1em plus 0.5em minus 0.4em\relax Springer, 2016.

\bibitem{goldman2004decentralized}
C.~V. Goldman and S.~Zilberstein, ``Decentralized control of cooperative systems: Categorization and complexity analysis,'' \emph{Journal of artificial intelligence research}, vol.~22, pp. 143--174, 2004.

\bibitem{pomdpssynthesis}
R.~Nair, P.~Varakantham, M.~Tambe, and M.~Yokoo, ``Synthesis of distributed constraint optimization and pomdps,'' in \emph{Proceedings of the AAAI}.

\bibitem{dobbe2017fully}
R.~Dobbe, D.~Fridovich-Keil, and C.~Tomlin, ``Fully decentralized policies for multi-agent systems: An information theoretic approach,'' \emph{arXiv preprint arXiv:1707.06334}, 2017.

\bibitem{moerland2023model}
T.~M. Moerland, J.~Broekens, A.~Plaat, C.~M. Jonker \emph{et~al.}, ``Model-based reinforcement learning: A survey,'' \emph{Foundations and Trends{\textregistered} in Machine Learning}, vol.~16, no.~1, pp. 1--118, 2023.

\bibitem{howard1960dynamic}
R.~A. Howard, \emph{Dynamic Programming and Markov Processes}.\hskip 1em plus 0.5em minus 0.4em\relax John Wiley \& Sons, 1960.

\bibitem{yu2019convergent}
M.~Yu, Z.~Yang, M.~Kolar, and Z.~Wang, ``Convergent policy optimization for safe reinforcement learning,'' \emph{Advances in Neural Information Processing Systems}, vol.~32, pp. 3127--3139, 2019.

\bibitem{bicego2003sequential}
M.~Bicego, V.~Murino, and M.~A. Figueiredo, ``A sequential pruning strategy for the selection of the number of states in hidden markov models,'' \emph{Pattern Recognition Letters}, vol.~24, no. 9-10, pp. 1395--1407, 2003.

\bibitem{lin2017runtime}
J.~Lin, Y.~Rao, J.~Lu, and J.~Zhou, ``Runtime neural pruning,'' in \emph{Proceedings of the 31st International Conference on Neural Information Processing Systems}, 2017, pp. 2178--2188.

\bibitem{tan2021towards}
K.~Tan and D.~Wang, ``Towards model compression for deep learning based speech enhancement,'' \emph{IEEE/ACM Transactions on Audio, Speech, and Language Processing}, vol.~29, pp. 1785--1794, 2021.

\bibitem{gupta2020learning}
M.~Gupta, S.~Aravindan, A.~Kalisz, V.~Chandrasekhar, and L.~Jie, ``Learning to prune deep neural networks via reinforcement learning,'' \emph{arXiv preprint arXiv:2007.04756}, 2020.

\bibitem{xu2021survey}
J.~Xu, W.~Zhou, Z.~Fu, H.~Zhou, and L.~Li, ``A survey on green deep learning,'' \emph{arXiv preprint arXiv:2111.05193}, 2021.

\end{thebibliography}

\end{document}